\documentclass[conference]{IEEEtran}
\IEEEoverridecommandlockouts

\usepackage{times}
\usepackage{amsmath}
\usepackage{amsfonts}
\usepackage{amsthm}
\usepackage{graphicx}
\usepackage[ruled]{algorithm2e}
\usepackage{algpseudocode}
\usepackage{multirow}
\usepackage{xcolor}
\usepackage{pagecolor}
\usepackage{booktabs}
\usepackage{cite}
\usepackage{svg}
\usepackage{hyperref}

\newtheorem{theorem}{Theorem}

\title{\LARGE \textbf{ManeuverGPT\\Agentic Control for Safe Autonomous Stunt
Maneuvers} }

\author{ Shawn Azdam$^{2,3,4}$, Pranav Doma$^{1,2}$, and Aliasghar Moj Arab $^{1,2,3}$%
\thanks{$^{1}$Department of Mechanical and Aerospace Engineering, New York
University, New York, NY 10012, USA.}%
\thanks{$^{2}$Agile Safe Autonomous Systems (ASAS) Lab, Tandon School of
Engineering, New York University, Brooklyn, NY 11201, USA.}%
\thanks{$^{3}$General Autonomy Inc., 201 Centennial Ave. Piscataway, NJ, 08854, USA.}%
\thanks{$^{4}$Azdam AI}%
\thanks{Emails: \texttt{shawn@azdam.com}, \texttt{pd2365@nyu.edu}, \texttt{aliasghar.arab@nyu.edu, mojarab@genauto.ai}.}%
}

\begin{document}
  \maketitle
  \thispagestyle{empty}
  \pagestyle{empty}

  \begin{abstract}
    The next generation of active safety features in autonomous vehicles should
    be capable of safely executing evasive hazard-avoidance maneuvers—akin to those
    performed by professional stunt drivers—to achieve high-agility motion at the
    limits of vehicle handling.

    This paper presents a novel framework, \emph{ManeuverGPT}, for generating
    and executing high-dynamic stunt maneuvers in autonomous vehicles using
    large language model (LLM)-based agents as controllers. We target aggressive
    maneuvers, such as J-turns, within the CARLA simulation environment and demonstrate
    an iterative, prompt-based approach to refine vehicle control parameters, starting
    \emph{tabula rasa} without retraining model weights. We propose an \emph{agentic}
    architecture comprised of three specialized agents (1) a \emph{Query
    Enricher Agent} for contextualizing user commands, (2) a \emph{Driver Agent}
    for generating maneuver parameters, and (3) a \emph{Parameter Validator
    Agent} that enforces physics-based and safety constraints. Experimental
    results demonstrate successful J-turn execution across multiple vehicle
    models through textual prompts that adapt to differing vehicle dynamics. We
    evaluate performance via established success criteria and discuss limitations
    regarding numeric precision and scenario complexity. Our findings
    underscore the potential of LLM-driven control for flexible, high-dynamic maneuvers,
    while highlighting the importance of hybrid approaches that combine
    language-based reasoning with algorithmic validation.
  \end{abstract}


  \section{Introduction}
  Autonomous vehicles (AVs) are advancing rapidly, offering potential benefits such
  as reduced traffic congestion, lower accident rates, and enhanced mobility.
  Integrating human-inspired active safety features derived from evasive hazard
  avoidance maneuvers, like those performed by professional stunt drivers,
  enables high-agility motion at the edge of handling limits to support the
  development of next-generation ``accident-free'' vehicles. Despite these advantages,
  automated driving still faces critical challenges in executing high-dynamic
  maneuvers under uncertain and varying conditions.

  One representative example is the stunt J-turn, a rapid $180^{\circ}$ rotation
  of the vehicle at speed as shown in Figure~\ref{fig:maneuver}. The safe
  execution of these maneuvers can expand the maneuverability envelope, enabling
  effective hazard evasion in emergency scenarios~\cite{arab2024high, arab2024safe}. However,
  designing a reliable controller to perform such a maneuver is challenging,
  due to the complex dynamics involved and the narrow margin for error. To address
  these challenges, we developed an agentic architecture that incorporates
  foundation models~\cite{bommasani2021opportunities} to assist in controllers
  for stunt maneuvers, named ManeuverGPT.

  Traditional approaches to such maneuvers often rely on carefully tuned controllers
  or reinforcement learning (RL) methods that require extensive training data,
  environment interactions, and domain-specific engineering~\cite{zhao2024adaptive,
  selvaraj2022ml}. Moreover, adapting to new vehicle dynamics can demand significant
  retraining or model redesign~\cite{chen2024deep, wang2018automated} . Recent advances
  in Large Language Models (LLMs) offer an alternative strategy for finding control
  policies for complex maneuvers without extensive retraining. Pre-trained LLMs
  have shown promising results in tasks such as planning, code synthesis, and
  robotic instruction following ~\cite{jiang2024self, quartey2024verifiably}.

  \begin{figure}[ht]
    \centering
    \includegraphics[width=\columnwidth]{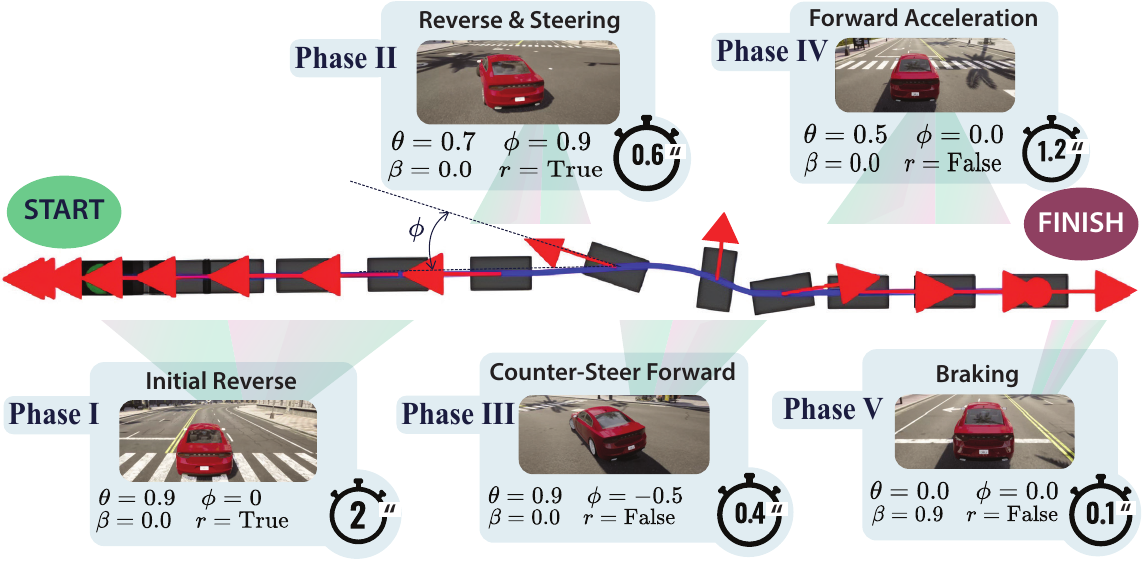}
    \caption{ The five‐phase J‐turn maneuver executed by our \emph{ManeuverGPT}
    controller. Phase~I begins with an initial reverse, followed by reverse
    steering (Phase~II), counter‐steering forward (Phase~III), forward
    acceleration (Phase~IV), and concludes with braking (Phase~V). Throttle~($\theta$),
    steering~($\phi$), brake~($\beta$), and reverse~($r$) control parameters are
    annotated throughout, culminating in a $180^{\circ}$ turnaround. }
    \label{fig:maneuver}
  \end{figure}

  In this work, we explore whether LLM agents can generate and refine \emph{control
  commands} for stunt driving maneuvers by iteratively adjusting key vehicle
  parameters and textual prompts to account for variations in dynamics and
  environmental conditions, starting \emph{tabula rasa}. We propose a novel framework,
  \emph{ManeuverGPT}, that incorporates three specialized LLM-driven agents in
  a closed-loop architecture.

  The main contributions of this work are threefold: (1) a multi-agent LLM-driven
  framework for generating stunt maneuvers without retraining model parameters,
  (2) a demonstration of generalization across multiple vehicle types in a
  simulator by adjusting only the textual prompts, and (3) a performance analysis
  using established success metrics (final orientation, collision checks, and
  time constraints).

  We also discuss limitations related to numeric precision and scenario complexity,
  which point to the need for hybrid methods that combine language-based
  reasoning with conventional control theory.

  The remainder of this paper is organized as follows. We review related work on
  LLMs in robotics and agentic architectures for autonomous vehicles. We then
  formally describe our framework, including its theoretical underpinnings and an
  iterative validation pipeline. Experimental results in CARLA demonstrate how
  prompt-based adjustments suffice to achieve feasible execution across
  different vehicle models. We conclude by highlighting open research directions,
  particularly in bridging simulation-to-reality gaps and ensuring robust
  safety guarantees for high-speed maneuvers.

  We focus on the J-turn maneuver~\cite{arab2021instructed,shahabi2021dynamic}~
  in the CARLA simulation environment~\cite{dosovitskiy2017carla} and evaluate how
  textual refinements alone (as opposed to gradient-based learning) can produce
  successful $180^{\circ}$ reorientations within specified time limits and with
  collision-free execution. The source code for our implementation is available
  at
  \href{https://github.com/SHi-ON/ManeuverGPT}{https://github.com/SHi-ON/ManeuverGPT}


  \section{Related Work}
  Recent advances in large language models (LLMs) have spurred research into
  their application for robotics control~\cite{zeng2023large}. Several works have
  demonstrated that LLMs can extract actionable knowledge for embodied agents,
  enabling zero-shot planning and reasoning in complex scenarios~\cite{huang2022language,
  kojima2022large}.

  LanguageMPC~\cite{sha2023languagempc} showed direct translation of linguistic
  decisions to model predictive control (MPC) parameters, reducing navigation costs
  compared to conventional controllers. The LLM4AD architecture~\cite{cui2024large},
  showed how natural language commands can be transformed into vehicle control parameters
  through structured reasoning processes—a conceptual approach our work extends
  to the domain of high-dynamic maneuvers.

  Traditional MPC techniques, as applied in autonomous driving systems~\cite{nguyen2023lateral},
  offer strong guarantees in terms of constraint handling and optimality, but
  they rely heavily on precise dynamic models that may not capture extreme
  driving behaviors. Arab~\cite{Arab2024Motion} demonstrated MPC's capability for
  extreme maneuvers through a sparse stable-trees algorithm, achieving high-agility
  maneuvers in 1/7-scale vehicle tests while maintaining safety via augmented
  stability regions. In contrast, deep RL methods have shown promising results
  for aggressive maneuvers such as drifting and J-turns~\cite{chen2024deep} ,
  yet they often lack formal safety guarantees.

  Emerging hybrid approaches combine linguistic reasoning with physical
  constraints. Chen et al.~\cite{chen2024async} developed AsyncDriver with
  decoupled 2Hz LLM/20Hz planner operation, reducing computational load by 63\%.
  Long et al.~\cite{long2024vlmmpc} integrated VLMs with MPC to improve safety
  margins by 38\% in adverse conditions through visual-language parameter generation.

  In summary, while each of these paradigms—LLM-based control, MPC, and RL—has
  its own advantages, our work combines the adaptability of LLMs with rigorous safety
  checks through phase-optimized parameter generation and multi-stage validation,
  advancing beyond existing hybrid approaches~\cite{sha2023languagempc,chen2024async}.


  \section{Methodology}

  Our framework, illustrated in Fig.~\ref{fig:agentic}, comprises three collaborative
  LLM-driven agents coordinated by a central unit:

  \begin{enumerate}
    \item \textbf{Query Enricher Agent} ($\mathcal{A}_{E}$): Interprets and augments
      user prompts, yielding an enriched query $E(Q)$ that includes domain-specific
      details. By incorporating historical data, environmental inputs, and operational
      constraints, it transforms vague instructions into detailed, actionable queries
      that capture necessary parameters for maneuver planning.

    \item \textbf{Driver Agent} ($\mathcal{A}_{D}$): Building on the enriched query,
      the Driver Agent generates a \emph{structured} sequence of maneuver parameters
      $P$. It produces candidate plans that align with the intended stunt
      maneuver while balancing precision and adaptability.

    \item \textbf{Parameter Validator Agent} ($\mathcal{A}_{V}$): Evaluates candidate
      parameters $P$ against a set of constraints $\mathcal{C}= \{C_{s}, C_{o}\}$
      where $C_{s}$ enforces safety requirements and $C_{o}$ checks operational
      limits. This agent either refines or rejects invalid parameters.
  \end{enumerate}
  A central orchestration unit manages agent interactions in a closed-loop process.
  If the Parameter Validator identifies issues, the system triggers re-evaluation
  through the Driver Agent or requests clarification from the Query Enricher.
  This feedback mechanism, based on the validation outcome and simulator results,
  ensures that all generated maneuvers consistently meet safety and performance
  standards.

  \begin{figure}[t]
    \centering
    \includegraphics[width=\columnwidth]{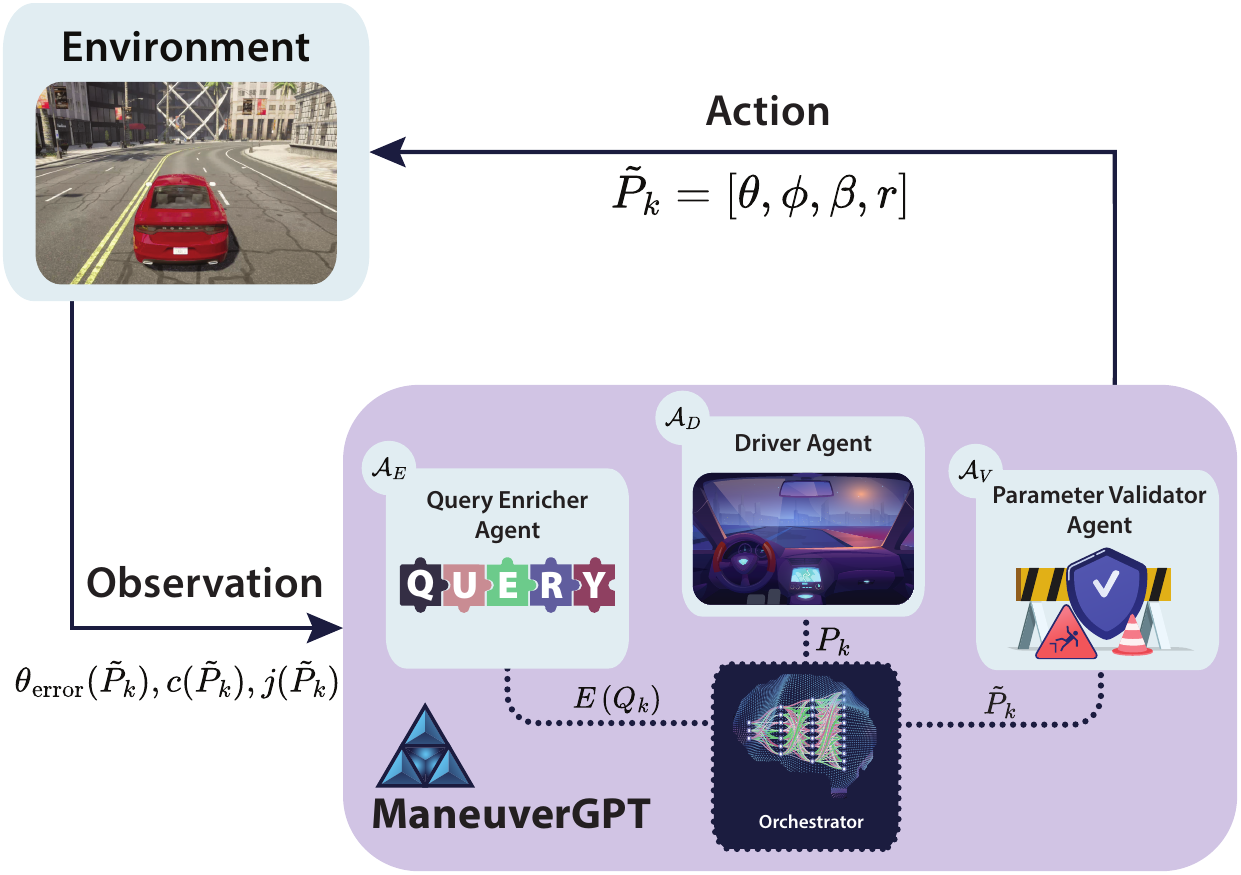}
    \caption{Overview of the proposed agentic framework. The system comprises three
    collaborative agents: the \textbf{Query Enricher Agent} ($\mathcal{A}_{E}$)
    for refining user commands, the \textbf{Driver Agent} ($\mathcal{A}_{D}$) for
    generating maneuver parameters, and the \textbf{Parameter Validator Agent} ($\mathcal{A}
    _{V}$) for enforcing safety and feasibility constraints. The \textbf{Orchestrator}
    coordinates agent interactions, forming a closed-loop process where
    observations guide iterative control refinements.}
    \label{fig:agentic}
  \end{figure}

  To formally capture the stunt maneuvers addressed by our agents, let the
  state of the autonomous vehicle at time $t$ be denoted by
  $\boldsymbol{x}(t) \in \mathcal{X}\subseteq \mathbb{R}^{n},$ where
  $\boldsymbol{x}$ is the state-space vector of dimension $n$ and $\mathcal{X}$
  is the feasible state space. Vehicle maneuvers are controlled through the input
  vector $\boldsymbol{u}(t)$ as
  \[
    \boldsymbol{u}(t) \;=\; [\theta(t), \phi(t), \beta(t), r(t)] \in \mathcal{U}
    ,
  \]
  where $\theta$ is the normalized throttle rate, $\phi$ is the normalized
  steering rate, $\beta$ is the normalized brake intensity, and $r$ indicates
  reverse status. The vehicle dynamics are represented as the nonlinear system
  \begin{equation}
    \boldsymbol{\dot{x}}(t) = \boldsymbol{f}\bigl(\boldsymbol{x}(t), \boldsymbol
    {u}(t)\bigr), \label{eq:vehicle_dynamics}
  \end{equation}
  with $\boldsymbol{f}(\cdot)$ capturing the kinematic, dynamic, and tire-road interaction
  effects studied in our earlier work~\cite{Arab2024Motion}. We assume
  $\boldsymbol{f}$ is continuous and locally Lipschitz in both $\boldsymbol{x}$
  and $\boldsymbol{u}$.

  The velocity components of state $\boldsymbol{x}(t)$ include longitudinal
  velocity $\boldsymbol{v}_{x}$, lateral velocity $\boldsymbol{v}_{y}$, and
  rotational velocity $\boldsymbol{\omega}$. These components follow coupled differential
  relationships that can be expressed as specific components of the general function
  $\boldsymbol{f}$~\cite{vehicle}:
  \begin{align}
    \dot{\boldsymbol{v}}_{x}  & = \boldsymbol{a}_{x}= \boldsymbol{f}_{v}(\theta, \beta, r) - \boldsymbol{\omega}\boldsymbol{v}_{y},         \\
    \dot{\boldsymbol{v}}_{y}  & = \boldsymbol{a}_{y}= \boldsymbol{f}_{l}(\phi, \boldsymbol{v}_{x}) + \boldsymbol{\omega}\boldsymbol{v}_{x}, \\
    \dot{\boldsymbol{\omega}} & = \boldsymbol{\alpha}= \boldsymbol{f}_{r}(\phi, \boldsymbol{v}_{x}, \boldsymbol{v}_{y}),
  \end{align}
  where $\boldsymbol{f}_{v}$, $\boldsymbol{f}_{l}$, and $\boldsymbol{f}_{r}$
  are the component functions of $\boldsymbol{f}$ that govern longitudinal, lateral,
  and rotational accelerations, respectively.

  A stunt maneuver is \emph{feasible} if it satisfies:
  \begin{itemize}
    \item \textbf{Physical Constraints:} Steering angles, throttle, brake, and reverse
      commands remain within manufacturer or simulation limits.

    \item \textbf{Safety Constraints:} No collisions occur during execution.

    \item \textbf{Performance Criteria:} The final state satisfies goal
      conditions (e.g., finishing a J-turn at $180^{\circ}\pm \Delta_{\theta}$).
  \end{itemize}

  We define a discrete iteration $k \in \{1, 2, \dots\}$ as follows:
  \begin{enumerate}
    \item Generate parameters $P_{k}= \mathcal{A}_{D}\bigl(E(Q_{k})\bigr)$.

    \item Validate $P_{k}$ to obtain $\tilde{P}_{k}$ via $\mathcal{A}_{V}$,
      ensuring that constraints are met.

    \item Run a simulation with $\tilde{P}_{k}$ to measure performance metrics (heading
      angle error, collisions, etc.).

    \item Produce feedback $Q_{k+1}$ based on the measured performance.
  \end{enumerate}
  This yields a feedback-driven sequence $\{Q_{k}\}$ and $\{P_{k}\}$.

  In each iteration, we evaluate how well $\tilde{P}_{k}$ meets stunt
  objectives via the cost function
  \[
    L(\tilde{P}_{k}) \;=\; \alpha_{1}\,\bigl|\theta_{\mathrm{error}}(\tilde{P}_{k}
    )\bigr| \;+\;\alpha_{2}\,c(\tilde{P}_{k}) \;+\;\alpha_{3}\,j(\tilde{P}_{k}),
  \]
  where $\theta_{\mathrm{error}}(\tilde{P}_{k})$ is the final heading deviation,
  $c(\tilde{P}_{k}) \in \{0,1\}$ indicates whether a collision occurred,
  $j(\tilde{P}_{k})$ is a measure of jerk (i.e., derivative of acceleration), and
  $\alpha_{1,2,3}\ge 0$ are weighting coefficients. Minimizing
  $L(\tilde{P}_{k})$ aims to achieve precise reorientation, avoid collisions,
  and maintain smooth maneuvers.

  Some reinforcement learning approaches define a reward $R(\tilde{P}_{k})$ to
  be maximized rather than a cost to be minimized. One possible definition is:
  \[
    R(\tilde{P}_{k}) = \alpha_{1}\,\bigl(180^{\circ}- \lvert \theta_{\mathrm{error}}
    (\tilde{P}_{k})\rvert\bigr) + \alpha_{2}\,(1 - c(\tilde{P}_{k})) - \alpha_{3}
    \,j(\tilde{P}_{k}),
  \]
  using the same weighting coefficients $\alpha_{1,2,3}$ as in the cost
  function. In this formulation, increasing $R$ is effectively equivalent to
  decreasing $L$. Our iterative algorithm remains cost-based but could be adapted
  to reward-based methods by taking $L \propto -R$.

  \begin{algorithm}
    [ht]
    \caption{ManeuverGPT}
    \label{alg:agentic} \DontPrintSemicolon \SetKwInOut{Input}{Input} \SetKwInOut{Output}{Output}

    \Input{ - User command $Q_{1}$ \\ - Constraints $\mathcal{C}= \{C_{s},C_{o}\}$ \\ - Maximum iterations $k_{\max}$ \\ - Cost threshold $\varepsilon$ }
    \Output{Feasible parameter set $\tilde{P}_{k}$ (or best-effort parameters)}
    \BlankLine

    \nl $k \gets 1$\; \nl $\tilde{P}_{\text{best}}\gets \emptyset$; $L_{\text{best}}
    \gets \infty$\; \nl \While{$k \le k_{\max}$}{ \nl $E(Q_{k}) \leftarrow \mathcal{A}_{E}(Q_{k})$ \; \nl $P_{k}\leftarrow \mathcal{A}_{D}\bigl(E(Q_{k})\bigr)$ \; \nl $\tilde{P}_{k}\leftarrow \mathcal{A}_{V}(P_{k})$ \; \nl \If{$L(\tilde{P}_{k}) \leq L_{\text{best}}$}{ \nl $\tilde{P}_{\text{best}}\gets \tilde{P}_{k}$; $L_{\text{best}}\gets L(\tilde{P}_{k})$\; } \nl \If{ $L(\tilde{P}_{k}) \le \varepsilon$ }{ \textbf{return} $\tilde{P}_{k}$ \tcp*{Satisfactory solution} } \nl \Else{ \nl $Q_{k+1}\gets \text{Feedback}(Q_{k}, \theta_{\text{error}}(\tilde{P}_{k}), c(\tilde{P}_{k}), j(\tilde{P}_{k}))$\; \nl $k \gets k+1$\; } }
    \nl \textbf{return} $\tilde{P}_{\text{best}}$ \tcp*{Best-effort solution}
  \end{algorithm}

  \begin{theorem}[Finite-Time Feasibility]
    \label{thm:finite_time_feasibility} Suppose that:
    \begin{enumerate}
      \item The validation operation $\mathcal{A}_{V}$ ensures $\tilde{P}_{k}$ remains
        in a compact set $\mathcal{U}_{\mathrm{safe}}$.

      \item The feedback reduces $L(\tilde{P}_{k})$ by at least a constant $\delta
        > 0$ whenever $L(\tilde{P}_{k}) > \varepsilon > 0$.

      \item The environment simulation is deterministic with respect to $(x_{0},
        \tilde{P}_{k})$.
    \end{enumerate}
    Then, there exists an integer $K$ such that, for all $k \ge K$, the parameter
    set $\tilde{P}_{k}$ is feasible (i.e., satisfies stunt requirements), or
    the user terminates the procedure after a finite number of iterations.
  \end{theorem}

  \begin{proof}
    Let $L(\tilde{P}_{k})$ be nonnegative. Each iteration reduces
    $L(\tilde{P}_{k})$ by at least $\delta$ when
    $L(\tilde{P}_{k}) > \varepsilon$. Since $L$ is bounded below by zero, a
    simple monotonicity argument shows it must reach feasibility (i.e., drop below
    the $\varepsilon$ threshold) within a finite number of steps, or else the user
    halts after some finite $k$.
  \end{proof}

  \noindent
  In our framework, the cost reduction guarantee in assumption 2 is realized through
  structured prompt refinement. When $L(\tilde{P}_{k}) > \varepsilon$, the Driver
  Agent ($\mathcal{A}_{D}$) receives specific feedback about performance gaps (e.g.,
  ``reduce steering angle by 5-10\% to minimize overshoot''). This targeted feedback,
  combined with the Parameter Validator's constraints, ensures progressive improvement
  in subsequent iterations. For example, prompt refinements might include:

  \begin{quote}
    \small \textit{Initial prompt:} ``Execute a J-turn maneuver.''\\ \textit{Refined
    prompt after feedback:} ``Execute a J-turn maneuver with gentler steering during
    phase 2 (currently overshooting by 12°) and increase brake intensity to 0.7
    during final stabilization.''
  \end{quote}

  \noindent
  This structure guarantees that each iteration addresses specific performance
  deficits, ensuring the $\delta$ improvement.

  \noindent
  \emph{Remark:} Theorem~\ref{thm:finite_time_feasibility} guarantees
  feasibility but not \emph{optimality} in the control-theoretic sense (e.g.,
  minimal torque or minimal time). It establishes that our iterative prompting
  approach can identify constraint-satisfying solutions under mild assumptions,
  which aligns with our goal of achieving safe executable maneuvers rather than
  mathematically optimal trajectories. This synergy between language-based
  generation and physically motivated cost functions is validated empirically
  in the next section.


  \section{Experimental Setup}
  Experiments were conducted in the CARLA simulator (v0.9.14), which provides
  high-fidelity vehicle dynamics and sensor modeling for testing complex
  maneuvers like J-turns.

  Performance was quantified using the following metrics: \textbf{Angle Error}
  is defined as the difference between the achieved turn angle and the ideal
  $180^{\circ}$ (Optimal: $\leq 3^{\circ}$, Acceptable: $\leq 10^{\circ}$);
  \textbf{Success Rate} denotes the percentage of trials with an angle error below
  $10^{\circ}$; \textbf{Yaw Rate} represents the angular velocity around the vertical
  axis (degrees per second); \textbf{Jerk} quantifies the rate of change of acceleration
  as a proxy measure of smoothness; \textbf{Steering Smoothness} is given by the
  inverse of the mean absolute yaw changes; \textbf{Execution Time} is the duration
  required to complete the maneuver; and \textbf{Collision Detection} is a binary
  metric assessing collision-free operation. These metrics provide a
  multifaceted evaluation of both technical performance and real-world
  applicability.

  Our system architecture integrates a GPT-family model as the core of the
  language agents. The agents exchange and share message context through an orchestrator
  component built on top of an in-memory database for asynchronous processing.
  We adopt a phase-based control protocol that plans complete trajectories
  rather than making frame-by-frame decisions. This approach enables high-level
  maneuver planning while maintaining computational efficiency and responsiveness.

  To ensure reliable and safe operation, we implemented multiple protective
  mechanisms throughout the experimental framework. Control parameter validation
  constrained all inputs within physical limits (throttle/brake: [0,1],
  steering: [-1,1]). Prompt-based safety constraints provided explicit
  instructions for vehicle stability during maneuvers. We implemented collision
  detection that terminates trials upon impact and provides negative feedback to
  the model. We also repeat the experiment from identical initial states across
  multiple runs to test iterative improvement.


  \section{Results and Discussion}
  \label{sec:results}

  Our investigation of LLM-based controllers for J-turn maneuvers revealed
  insights across vehicle dynamics, architectural implementations, and cross-vehicle
  adaptability. The controller achieved an overall 86\% implementation success
  rate across 100 runs, improving from 83.3\% in early trials to 90\% in later trials
  (Table~\ref{tab:veh}).

  \begin{table}[ht]
    \centering
    \caption{Parameter Execution Performance (Averaged over 100 Runs)}
    \label{tab:veh} \scriptsize
    \resizebox{\columnwidth}{!}{
    \begin{tabular}{|l|p{1.5cm}|p{1.5cm}|p{1.5cm}|p{1.5cm}|}
      \hline
      \textbf{Batch}   & \textbf{Total Parameters} & \textbf{Implemented} & \textbf{Rejected} & \textbf{Success (\%)} \\
      \hline
      Overall          & 100                       & 86                   & 14                & 86.0\%                \\
      \hline
      Early (first 60) & 60                        & 50                   & 10                & 83.3\%                \\
      \hline
      Later (last 40)  & 40                        & 36                   & 4                 & 90.0\%                \\
      \hline
    \end{tabular}
    }
  \end{table}

  \begin{figure}[t]
    \centering
    \includegraphics[width=\columnwidth]{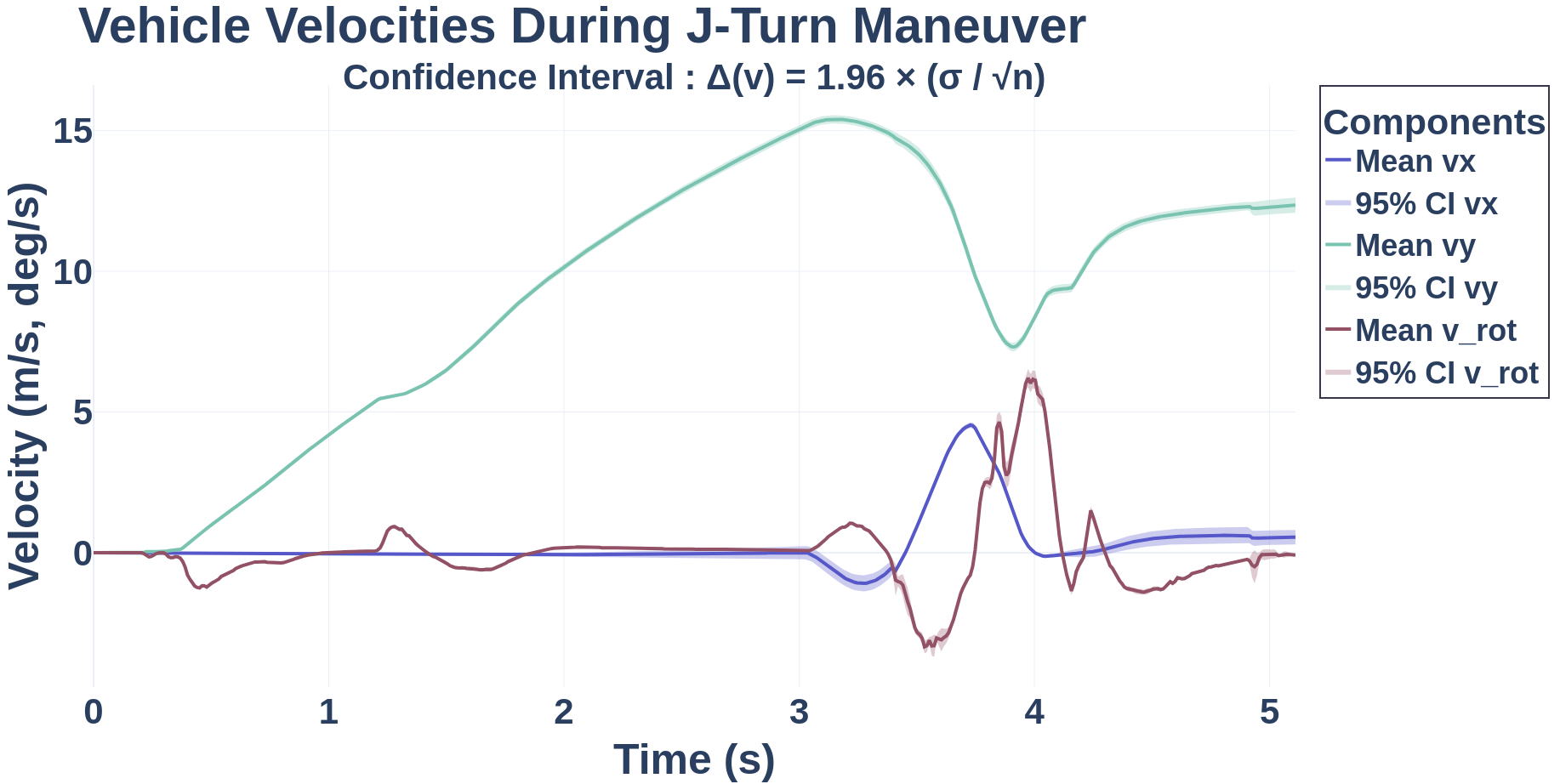}
    \caption{Time series of vehicle velocities during a J-turn maneuver, showing
    longitudinal ($v_{x}$ in m/s), lateral ($v_{y}$ in m/s), and rotational ($v_{rot}$
    in deg/s) velocity components with their respective 95\% confidence
    intervals (CI). The CI range is computed as $\Delta(v) = 1.96 \times (\sigma
    / \sqrt{n})$, where $\sigma$ is the standard deviation and $n$ is the
    number of trials.}
    \label{fig:vehicle_velocities}
  \end{figure}

  The velocity profile in Figure~\ref{fig:vehicle_velocities} reveals key insights
  about J-turn execution. During the acceleration phase (0-1.5s), the
  controller prioritizes longitudinal acceleration to approximately 15~m/s
  while maintaining minimal lateral and rotational velocities, demonstrating understanding
  that forward momentum is necessary before turning. In turn initiation (1.5-3s),
  we observe gradual lateral velocity development, indicating controlled
  steering without destabilizing abrupt changes. The maximum rotation phase (3-4s)
  shows peak rotational velocity (approximately 6~deg/s) coinciding with significant
  longitudinal velocity drop and maximum lateral velocity—representing the core
  directional change. The stabilization phase (4-5s) exhibits damping
  oscillations as the controller stabilizes the vehicle, with longitudinal velocity
  recovering to approximately 12~m/s. Narrow confidence intervals for
  longitudinal and lateral velocities indicate consistent performance, while wider
  intervals for rotational velocity suggest variability due to traction conditions
  and control timing differences.

  Our multi-agent architecture substantially outperformed a single-agent implementation.
  For this architectural comparison, we employed a more stringent intermediate threshold
  of $7^{\circ}$ (rather than the $10^{\circ}$ used in vehicle comparisons), as
  both implementations tested the same vehicle type (sedan). The multi-agent system
  maintained error below $7^{\circ}$ for 76\% of trials, compared to only 52\%
  with the single-agent system. Furthermore, the multi-agent system achieved
  optimal performance (less than $3^{\circ}$) in 31\% of trials, versus 18\% for
  the single-agent approach.

  \begin{figure}[t]
    \centering
    \includegraphics[width=0.5\textwidth]{
      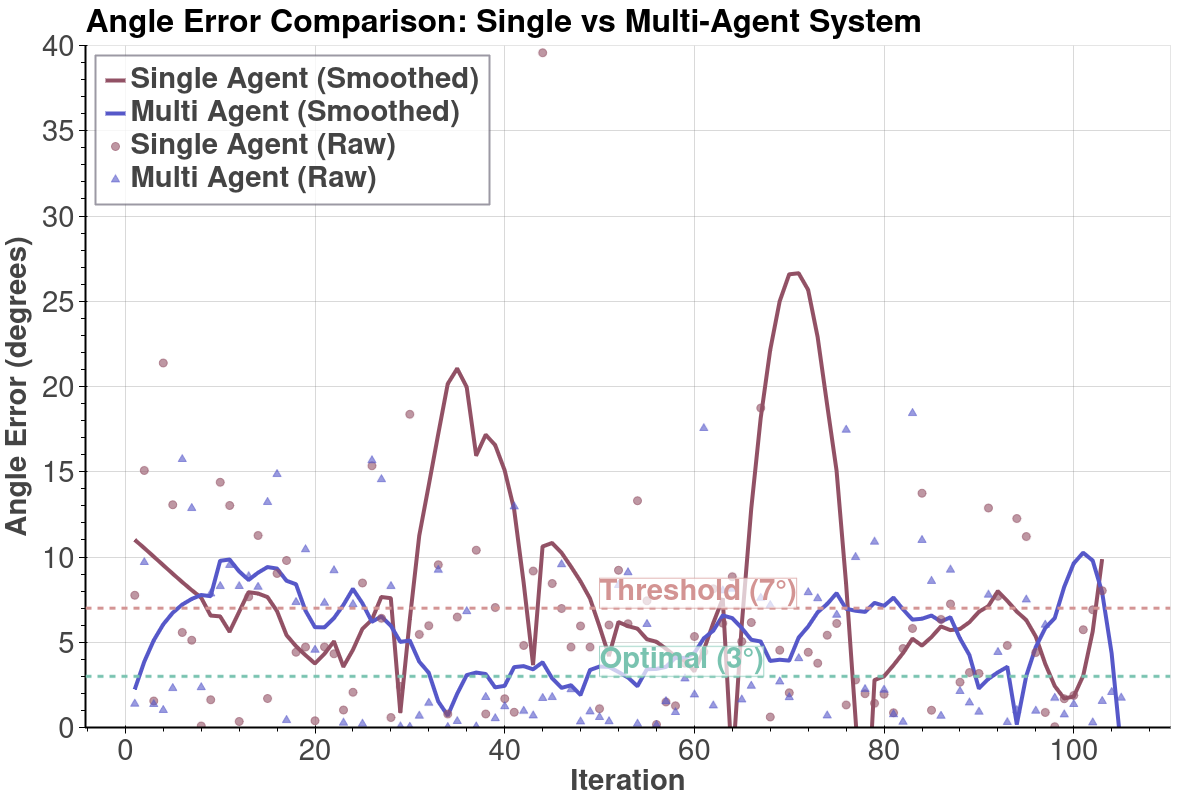
    }
    \caption{Angle error comparison between single-agent and multi-agent
    systems. The raw and smoothed angle errors are plotted for both systems across
    multiple trials. The $7^{\circ}$ threshold (dashed purple line) represents
    the acceptable error limit, while the $3^{\circ}$ optimal error (dashed
    green line) indicates the desired accuracy.}
    \label{fig:angle_error_single_vs_multi_agent}
  \end{figure}

  As shown in Figure~\ref{fig:angle_error_single_vs_multi_agent}, appropriate task
  decomposition enhances control precision and reliability. The smoothed trend
  lines indicate that the multi-agent system maintains more stable performance
  over time, while the single-agent system shows greater oscillation and some regression
  in later trials. Because the Validator imposes guardrails that reduce the
  single‐agent’s tendency to drift in parameter selection. For instance, we
  observed that without the Validator, the single-agent sometimes picks steering
  angles >1.0, leading to immediate collisions.

  The controller also demonstrated varying effectiveness across different vehicle
  dynamics. Table~\ref{tab:vehicle_comparison} shows that the sedan
  consistently outperformed the sports coupe across nearly all metrics, with 62.7\%
  lower mean angle error (9.09° vs. 21.39°) and 29.5\% higher success rate (90.11\%
  vs. 70.57\%).

  \begin{table}[ht]
    \centering
    \caption{Performance Comparison Between Sedan and Sports Coupe Models}
    \label{tab:vehicle_comparison} \scriptsize
    \begin{tabular}{lrr}
      \toprule \textbf{Metric}      & \textbf{Sedan} & \textbf{Sports Coupe} \\
      \midrule Mean Angle Error (°) & 9.09           & 21.39                 \\
      Median Angle Error (°)        & 5.35           & 6.99                  \\
      Min Angle Error (°)           & 0.00           & 0.12                  \\
      Max Angle Error (°)           & 179.98         & 178.78                \\
      Standard Deviation            & 24.04          & 44.24                 \\
      Success Rate (\%)             & 90.11          & 70.57                 \\
      Mean Jerk (m/s$^{3}$)         & 0.81           & 1.26                  \\
      Avg Max Jerk (m/s$^{3}$)      & 48.31          & 55.32                 \\
      Mean Yaw Rate (°/s)           & 16.14          & 13.82                 \\
      Steering Smoothness           & 0.39           & 0.44                  \\
      Avg Execution Time (s)        & 264.00         & 292.17                \\
      \bottomrule
    \end{tabular}
  \end{table}

  \begin{figure}[htbp]
    \centering
    \includegraphics[width=0.5\textwidth]{
      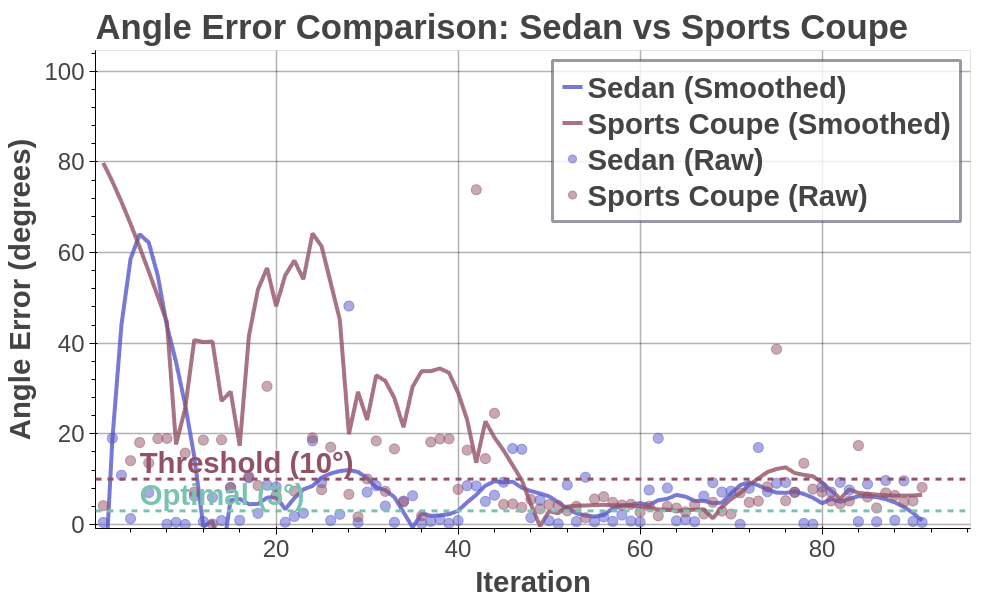
    }
    \caption{Comparison of angle error between sedan and sports coupe vehicle
    models during J-turn maneuvers. The angle error is calculated as the absolute
    difference between the actual turn angle and the ideal 180-degree turn.
    Lower values indicate better performance, with 3° considered optimal (green
    line) and 10° as the maximum acceptable threshold (red line). The smoothed
    lines represent the trend using a Savitzky-Golay filter, while individual data
    points show raw measurements. Angle error
    $\theta_{e}= |\theta_{\text{final}}- \theta_{\text{initial}}- 180^{\circ}|$
    consistently achieves sub-10° precision in the simulation environment.}
    \label{fig:angle_error_sedan_vs_sports_coupe}
  \end{figure}

  Figure~\ref{fig:angle_error_sedan_vs_sports_coupe} shows that both vehicles achieved
  near-perfect turns in their best trials, but the sports coupe showed
  substantially higher variability (SD: 44.24 vs. 24.04) despite also nearing
  optimal thresholds in later iterations. The sedan maintained performance below
  the $10^{\circ}$ threshold more consistently, while the sports coupe
  exhibited more frequent excursions beyond acceptable limits.

  \begin{figure}[t]
    \centering
    \includegraphics[width=\columnwidth]{
      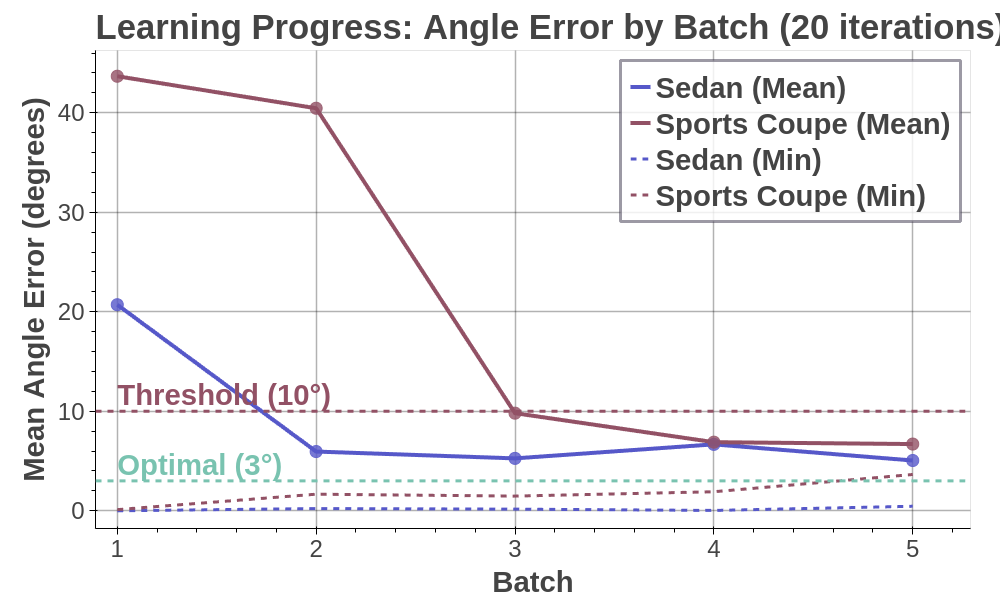
    }
    \caption{Learning progress of the steering controller for sedan and sports
    coupe vehicles across batches of 20 iterations. The plot shows the mean angle
    error (solid lines) and minimum angle error (dashed lines) for each batch,
    demonstrating how controller performance improves over time. The horizontal
    dashed lines at 3° and 10° represent the optimal and maximum acceptable error
    thresholds, respectively. }
    \label{fig:angle_error_learning_progress}
  \end{figure}

  Figure~\ref{fig:angle_error_learning_progress} shows the learning progress over
  5 batches (epochs) of 20 iterations. The controllers' performance gradually converges
  toward the optimal error level (3°) and remains below the 10° threshold as training
  progresses. Each iteration takes roughly five seconds using API calls, making
  the learning process remains practical online.

  These differences highlight a key insight: while our multi-agent architecture
  can adapt to different vehicle dynamics, its effectiveness varies based on inherent
  stability characteristics. The sedan's more forgiving dynamics allow greater
  error margins in control parameters, while the sports coupe's higher responsiveness
  amplifies small control errors into larger outcome differences.

  The sports coupe's shorter wheelbase, higher power-to-weight ratio, and rear-biased
  weight distribution make it more responsive but also more challenging to control
  precisely during high-dynamic maneuvers, explaining the performance disparity
  despite the controller's adaptive capabilities.


  \section{Conclusion}
  We have presented \emph{ManeuverGPT} , a multi-agent framework that utilizes
  LLMs for generating and refining high-dynamic stunt maneuvers such as J-turns
  in the CARLA simulator. Our research demonstrates that LLM-based controllers can
  effectively plan and execute complex vehicle maneuvers through iterative feedback
  without requiring internal parameter modifications. Our findings reveal that:
  (1) multi-agent architectures outperform single-agent implementations by 46\%
  in optimal execution rate; (2) the controller adapts to different vehicle dynamics,
  achieving 90.11\% success with sedans and 70.57\% with more challenging
  sports coupes; and (3) performance improves through structured feedback, with
  implementation success increasing from 83.3\% to 90\% over successive
  iterations.

  While our prompt-based approach enables the execution of complex maneuvers
  \emph{tabula rasa} and avoids retraining model weights, several challenges
  remain. The current implementation lacks formal safety guarantees that conventional
  MPC methods provide, and precise numeric control would benefit from hybrid
  approaches combining LLM reasoning with algorithmic optimization. Additionally,
  scaling to more complex traffic scenarios and bridging the simulation-to-reality
  gap both present significant hurdles.

  Future work should integrate explicit safety constraints into the LLM prompt
  context, develop automated re-prompting based on sensor data, and explore
  synergies between language-based high-level planning and specialized models
  for reactive control. This research establishes LLM-driven control as a promising
  approach for rapid prototyping and novel maneuver development when paired with
  appropriate validation frameworks, potentially expanding the envelope of
  autonomous vehicle capabilities for safety-critical evasive maneuvers.


  \section*{Acknowledgments}
  The authors would like to thank Dr.\ Bahram Behzadian for his invaluable
  insights and critical feedback on this manuscript. His thorough review and thoughtful
  suggestions greatly contributed to the quality of this work.

  \bibliographystyle{IEEEtran}
  \bibliography{references}
\end{document}